\documentclass[review]{elsarticle}

\usepackage{caption}
\usepackage{subcaption}
\usepackage{algorithmic}
\usepackage{algorithm}
\usepackage{graphicx}
\usepackage{textcomp}
\usepackage{xcolor}
\usepackage{bbm}
\usepackage{algorithm}
\usepackage{caption}
\usepackage{subcaption}
\usepackage{enumitem}
\usepackage{booktabs}
\usepackage{amsfonts}
\usepackage{amsmath}
\usepackage{hyperref}
\usepackage{multirow}
\usepackage{mathtools}

\DeclareSymbolFont{matha}{OML}{txmi}{m}{it}
\DeclareMathSymbol{\varv}{\mathord}{matha}{118}

\journal{Journal of \LaTeX\ Templates}

\newcommand{\independent}{\perp\mkern-9.5mu\perp}

\newtheorem{definition}{Definition}[section]

\newtheorem{assumption}{Assumption}[section]
\newtheorem{theorem}{Theorem}[section]

\newproof{Example}{Example}
\newproof{Method}{Method}
\newproof{Exercise}{Exercise}
\newproof{proof}{Proof}

\begin{document}

\begin{frontmatter}

\title{Stochastic Intervention for Causal Inference via Reinforcement Learning}


\author[mymainaddress]{Tri Dung Duong\corref{mycorrespondingauthor}}
\cortext[mycorrespondingauthor]{Corresponding author}
\ead{TriDung.Duong@student.uts.edu.au}
\address[mymainaddress]{Faculity of Engineering and Information Technology, University of Technology Sydney, \\NSW, Australia}

\author[mymainaddress]{Qian Li}
\ead{Qian.Li@uts.edu.au}

\author[mymainaddress]{Guandong Xu}
\ead{Guandong.Xu@uts.edu.au}





\begin{abstract}
Causal inference methods are widely applied in various decision-making domains such as precision medicine, optimal policy and economics. 
Central to causal inference is the treatment effect estimation of intervention strategies, such as changes in drug dosing and increases in financial aid. 
Existing methods are mostly restricted to the deterministic treatment and compare outcomes under different treatments. However, they are unable to address the substantial recent interest of treatment effect estimation under stochastic treatment, e.g., ``how all units health status change if they adopt 50\% dose reduction''.  
In other words, they lack the capability of providing fine-grained treatment effect estimation to support sound decision-making.
In our study, we advance the causal inference research by proposing a new effective framework to estimate the treatment effect on stochastic intervention.
Particularly, we develop a stochastic intervention effect estimator (SIE) based on nonparametric influence function, with the theoretical guarantees of robustness and fast convergence rates. Additionally, we construct a customised reinforcement learning algorithm based on the random search solver which can effectively find the optimal policy to produce the greatest expected outcomes for the decision-making process. Finally, we conduct an empirical study to justify that our framework can achieve significant performance in comparison with state-of-the-art baselines.


\end{abstract}


\end{frontmatter}

\section{Introduction}
 Causal inference aims at estimating the causal effects of an intervention or treatment on an outcome, which increasingly plays a vitally important role in scientific investigations and real-world applications.  
    A widely used example is the causal effect for a binary treatment, in which the expectation of the outcome in a hypothetical world in which everybody receives treatment is compared with its counterpart in a world in which nobody does\footnote{\emph{Treatment} and \emph{outcome} are terms in the theory of causal inference, which for example denote a promotion strategy taken and its resulting profit, respectively}.
    Other examples include ``What is the effect of sleep deprivation on health outcomes?'' and  ``How would family socio-economic status affect career prospects?''.  Therefore, it is of great interest to develop models that can correctly predict the optimal treatment based on given subject characteristics. Treatment effect estimation can address this by comparing outcomes under different treatments.  

 Estimating treatment effect is challenging, because only the factual outcome for a specific treatment assignment (say, treatment \texttt{A}) is observable, while the counterfactual outcome corresponding to alternative treatment \texttt{B} is usually unknown.
    Aiming at deriving the absent counterfactual outcomes, existing causal inference from observations methods can be categorized into these main branches: re-weighting methods \cite{gruber2011tmle, austin2015moving}, tree-based methods \cite{chipman2007bayesian, hill2011bayesian, wager2018estimation}, matching methods \cite{rosenbaum1983central, dehejia2002propensity,stuart2011matchit} and doubly robust learners \cite{econml, 10.5555/3104482.3104620}. 
    In general, the matching approaches focus on finding the comparable pairs based on distance metrics such as propensity score or Euclidean distance, while re-weighting methods assign each unit in the population a weight to equate groups based on the covariates. Meanwhile, tree-based machine learning models including decision tree or random forest are utilized in the tree-based approach to derive the counterfactual outcomes. Doubly Robust Learner is another recently developed approach that combines propensity score weighting with the regression outcome to produce an unbiased and robust estimator. 
    
    

    
    However, recent studies in treatment effect estimation mainly focus on the deterministic intervention which sets each individual a deterministic treatment, incapable of dealing with dynamic and stochastic intervention~\cite{dudik2014doubly, pearl2000models, tian2012identifying}. i.e.,  the treatment is deterministic.  In many real-world applications, however, the effect of a stochastic intervention might be of interest,
    For example, rather than ``if we do not use the medicine treatment \texttt{A} for all units, what is resulting in health status (the desired outcome)?'', the medical researcher is more eager to know  ``how all units' health status change if we adopt 50\% dose reduction in medicine treatment \texttt{A}''.  
    In this case, the treatment variable is no longer deterministic but a stochastic value, and traditional causal inference methods fail to capture the stochastic intervention on the treatment variable. 
    

    To address these issues, we propose a novel influence function based model to provide sufficient causal evidence to answer decision-making questions about stochastic interventions. 
    Stochastic intervention estimation in our method can provide a fine-grained treatment effect estimation to gradually quantify the effect of the stochastic intervention on the outcomes. 
    In addition, we exploit stochastic intervention optimization to  customize stochastic intervention assignment, i.e., what is the best degree of intervention on the treatment to achieve the desired outcome.
    The main contributions of our work are summarized below:
    
\begin{itemize}
    \item 
    We propose a causal inference framework with stochastic intervention to learn the treatment effect trajectory, which tackles the limitation of existing approaches only dealing with deterministic intervention effects. Particularly, our framework introduces the concept of stochastic propensity score, and develops a semi-parametric influence function to learn stochastic intervention effect.   
    \item 
    Based on the general efficiency theory, we theoretically analyze the asymptotic behavior of our semi-parametric influence function. We prove that our influence function can achieve double robustness and fast parametric convergence rates. We also empirically demonstrate the effectiveness of the proposed influence function. 
    \item Based on the stochastic treatment effect estimation, our framework is capable of customizing the stochastic intervention, with the goal of uplifting desired outcomes on downstream decision-making applications.    
    We formulate the stochastic intervention optimization as a derivative-free optimization problem and design a random search solver to efficiently achieve the optimal expected outcome.
    

    
\end{itemize}



\section{Related works}
\label{section:related}
Conventionally, causal inference can be trickled by either the randomized experiment (also known as A/B testing in online settings) or observational data. 
    In randomized experiment, units are randomly assigned to a treatment and their outcomes are recorded. One treatment is selected as the best among the alternatives by comparing the predefined statistical criteria. While randomized experiments have been popular in traditional causal inference, it is prohibitively expensive \cite{chan2010evaluating, kohavi2011unexpected} and infeasible \cite{bottou2013counterfactual} in some real-world settings.
    As an alternative method, observational study is becoming increasingly critical and available in many domains such as medicine, public policy and advertising. 
    However, observational study needs to deal with data absence problem, which differs fundamentally from supervised learning. This is simply because only the factual outcome (symptom) for a specific treatment assignment (say, treatment \texttt{A}) is observable, while the counterfactual outcome corresponding to alternative treatment \texttt{B} in the same situation is always unknown. In the context of binary treatment, the individuals given the treatment are the treated group, whereas other individuals in the population are the control group. 
    
\subsection{Treatment Effect Estimation} 
The simplest way to estimate treatment effect in observational data is the matching method that finds the comparable units in the treated and controlled groups. The prominent matching methods include Propensity Score Matching (PSM) \cite{rosenbaum1983central, dehejia2002propensity} and Nearest Neightbor Matching (NNM) \cite{stuart2011matchit}. Particularly, for each treated individual, PSM and NNM select the nearest units in the controlled group based on some distance functions, and then calculate the difference between two paired outcomes. Another popular approach is reweighting method that involves building a classifier model to estimate the probability of a treatment assigned to a particular unit, and uses the predicted score as the weight for each unit in dataset. TMLE~\cite{gruber2011tmle} and IPSW~\cite{austin2015moving} fall into this category. Ordinary Linear Regression (OLS)~\cite{goldberger1964econometric} is another commonplace method that fits two linear regression models for the treated and controlled group, with each treatment as the input features and the outcome as the output. The predicted counterfactual outcomes thereafter are used to calculate the treatment effect. Meanwhile, decision tree is a popular non-parametric machine learning model, attempting to build the decision rules for the regression and classification tasks. Bayesian Additive Regression Trees (BART) \cite{chipman2007bayesian, hill2011bayesian} and Causal Forest \cite{wager2018estimation} are the prominent tree-based method in causal inference. While BART~\cite{chipman2007bayesian, hill2011bayesian} builds the decision tree for the treated and controlled units, Causal Forest~\cite{wager2018estimation} constructs the Random Forest model to derive the counterfactual outcomes, and then calculates the difference between the paired potential outcomes to obtain the average treatment effect. They are proven to obtain the more accurate treatment effect than matching methods and reweighting methods in the non-linear outcome setting.
Doubly Robust Learner \cite{econml, 10.5555/3104482.3104620} is the recently proposed approach that constructs a regression estimator predicting the outcome based on the covariates and treatment, and builds a classifier model to fit the treatment. DRL finally combines both predicted propensity score and predicted outcome to estimate treatment effect.

\subsection{Stochastic Intervention Optimization}
Our work focuses on estimating the intervention effect and thus finding the optimal intervention to maximize the expected outcomes in the population. This is closely related to the uplift modelling studies, with the goal of uplifting (or maximizing) the outcome with the treatment as compared to the outcome without the treatment\cite{zaniewicz2013support, Personalized_Medicine, hansotia2002incremental, manahan2005proportional}. 
Among the uplifting models, the most popular and widely-used approach is Separate Model Approach (SMA) \cite{zaniewicz2013support,Personalized_Medicine}.
SMA is applicable to binary treatment and builds two regression models under each treatment, respectively. The treatment with the best predictive outcome is chosen and defined as the optimal one. The advantage of SMA lies in its easy implementation since SMA does not require a specific machine learning algorithm.
Several state-of-the-art machine learning algorithms such as Random Forest, Gradient Boosting Regression or Adaboost can be applied to these two regression models \cite{liaw2002classification, solomatine2004adaboost, friedman2001greedy}. SMA has been widely applied in marketing \cite{hansotia2002incremental} and customer segmentation \cite{manahan2005proportional}. However, when dealing with the data containing a great deal of noisy and missing information, the model outcomes are prone to be incorrect and biased, which leads to poor performance. Other commonplace methods for uplift modelling include Class Transformation Model \cite{jaskowski2012uplift} and Uplift Random Forest \cite{guelman2014package}; these techniques however only deal with the binary outcome, so we do not discuss them here. 

\section{Preliminaries and Problem Definition}
\label{section:pre}
\subsection{Notation}


In this study, we consider the observational dataset $Z=\{\boldsymbol{x}_i,{y}_i,t_i\}_{i=1}^n$ with $n$ units, where $\boldsymbol{x}\in\mathbb{R}^{n\times d}$ is the $d$-dimensional covariate, 
$y$ and $t\in\{0,1\}$ are the outcome and the treatment for the unit, respectively.
The treatment variable is binary in many cases, thus the unit will be assigned to the control treatment if $t=0$, or the treated treatment if $t=1$. As a result, $y_0$ and $y_1$ are the potential outcomes corresponding to the control and treated units. 
According to the Rubin-Neyman causal model~\cite{imbens2015causal}, 
two potential outcomes $y_0(\boldsymbol{x})$ and $y_1(\boldsymbol{x})$
exist for $\boldsymbol{x}$ with the treatment $t=0$ and $t=1$, respectively. It is noted that either $y_0$ or $y_1$ can be observed for each subject in the population. 

After introducing the observational data and the key terminologies, the central goal of causal inference, i.e., treatment effect estimation, can be quantitatively definitions using the above definitions. To make the definition clear, here we define the treatment
effect under binary treatment. At the population level, the treatment effect is named as the Average Treatment Effect (ATE), which is defined as
\begin{equation}
    \tau_{\text{ATE}}=\mathbb{E}[y_0(\boldsymbol{x})-y_1(\boldsymbol{x})]
    \label{eq:tau}
\end{equation}
For causal inference, our objective is to estimate the treatment effects from the observational data. Based on the estimated ATE, different treatment conditions can be selected and applied to users to achieve preferred outcome.

To further illustrate the treatment effect estimation, we take an online marketing scenario for example. 
We denote each customer as a high-dimensional vector of features $\boldsymbol{x}$. The customer indexed by $i$ receives a promotion treatment $t_i\in\{0,1\}$. Accordingly, $y_0(\boldsymbol{x}_i)$ and $y_1(\boldsymbol{x}_i)$ are profit accrued from customer $i$ corresponding to either control treatment or treated treatment.
The effectiveness of a promotion campaign can be evaluated by computing average treatment effect of the promotion treatment on the customers.

\subsection{Propensity Score}
Rosenbaum and Rubin~\cite{rosenbaum1983central} first proposed propensity score technique to deal with the high-dimensional covariates. The propensity score is widely used in causal inference methods to estimate treatment effects from observational data~\cite{hirano2003efficient, pirracchio2016propensity, luo2010applying,abdia2017propensity}. This is largely because propensity score can help eliminating the great portion of bias, leading to a more balanced dataset and thus allowing a simple and direct comparison between the treated and untreated individuals. Particularly, propensity score can summarise the mechanism of treatment assignment and thus squeezes covariate space into one dimension to avoid the possible data sparseness issue~\cite{bang2005doubly, dehejia2002propensity, austin2015moving, hirano2003efficient}.  The propensity score is defined as the probability that a unit is assigned to a particular treatment $t=1$ given the covariate $\boldsymbol{x}$, i.e.,
\begin{equation}
    p_t(\boldsymbol{x}) = \mathbb{P}(t = 1 | \boldsymbol{x})
\end{equation}


In practice, one widely-adopted parametric model for estimating propensity score $p_t(\boldsymbol{x})$ 
is the logistic regression
\begin{equation}
    \hat{p}_t(\boldsymbol{x})=\frac{1}{1+
\exp{(\boldsymbol{w}^{\top}\boldsymbol{x}+\omega_0)}}
\label{eq:ps}
\end{equation} 
where $\boldsymbol{w}$ and $\omega_0$ are parameters estimated 
by minimizing the negative log-likelihood~\cite{martens2008systematic}.

\subsection{Assumption}\label{sc:asp} Following the general practice in causal inference literature \cite{pearl2010causal, pearl2003statistics, scheines1997introduction}, we consider the following two assumptions to ensure the identifiability of the treatment effect, i.e. \emph{Positivity} and \emph{Ignorability}. 

\begin{assumption} [Positivity] Each unit has a positive probability to be assigned by a treatment, i.e.,
\begin{equation}
    p_t(\boldsymbol{x}) >0, \quad\forall \boldsymbol{x}\text{ and } t
\end{equation}
\end{assumption}



\begin{assumption}[Ignorability]
The assignment to the treatment $t$ is independent of the outcomes $\boldsymbol{y}$ given covariates $\boldsymbol{x}$
\begin{equation}
   y_1 , y_0  \independent t|\boldsymbol{x}
\end{equation}
\end{assumption}

\section{Stochastic Intervention Effect}
\label{section:estimation}




Recall the goal of causal inference is to compute the treatment effect estimation that can be evaluated by the metric in Eq.~\eqref{eq:tau}. Namely, treatment effect estimation can be expressed by the difference between the observed outcome and the counterfactual outcome under a intervention on the treatment. Apparently, the observed outcome in the dataset is generated by the observed treatment (e.g., $t=1$). By contrast, the counterfactual outcome is generated by intervening the treatment, e.g., shifting treatment from observed $t=1$ to counterfactual $t=0$, which is however unobserved in practice.
Thus, intervention effect estimation is turned into a problem of predicting the counterfactual outcome generated by an intervention on the treatment. 

\subsection{Stochastic Counterfactual Outcome}
Before predicting the counterfactual outcome, we first propose stochastic propensity score to characterize the stochastic intervention.
 




\begin{definition}[Stochastic Propensity Score]{ The stochastic propensity score with respect to stochastic degree $\delta$ is 
\begin{equation}
    q_{t}( \boldsymbol{x},\delta) = \frac{\delta \cdot\hat{p}_t(\boldsymbol{x})}{\delta \cdot\hat{p}_t(\boldsymbol{x}) + 1 - \hat{p}_t(\boldsymbol{x})}
    \label{eq:incrps}
\end{equation}
and $\hat{p}_t(\boldsymbol{x})$ is denoted by
}
\begin{equation}
    \hat{p}_{t}(\boldsymbol{x})=\frac{\exp \left(\sum_{j=1}^{s} \beta_{j} g_{j}\left(\boldsymbol{x}\right)\right)}{1+\exp \left(\sum_{j=1}^{s} \beta_{j} g_{j}\left(\boldsymbol{x}\right)\right)}
    \label{eq:ourps}
\end{equation}
\label{def:ips}
where $\{ g_1, \cdots, g_s\}$ are nonlinear basis functions.
\end{definition}


The proposed stochastic propensity score in Definition \ref{def:ips} has two promising properties compared with classical propensity score in \eqref{eq:ps}. 
On the one hand, classical propensity score in \eqref{eq:ps} fails to quantify the causal effect under stochastic intervention. 
So we introduce $\delta$ in \eqref{eq:incrps} to represent the stochastic intervention indicating the extent to which the propensity scores are fluctuated from their actual observational values. 
For instance, the stochastic intervention that the doctor adopts 50\% dose increase in the patient can be expressed by $\delta = 1.5$. 
On the other hand, if there are higher-order terms or non-linear trends among covariates $\boldsymbol{x}$, classical propensity score using $\boldsymbol{w}^{\top}\boldsymbol{x}+\omega_0$ in Eq.~\eqref{eq:ps} may lead to misspecification ~\cite{dalessandro2012causally} .
So we propose to use a sum of nonlinear function $\sum_{j=1}^{s} \beta_{j} g_{j}$ in  \eqref{eq:ourps} that captures the non-linearity involving in covariates to create an unbiased estimator of treatment effect.

On the basis of the stochastic propensity score, we propose an influence function specific to estimate counterfactual outcome under stochastic intervention. 
Meanwhile, we also analyze the asymptotic behavior of the counterfactual outcome with theoretical guarantees. We prove that our influence function can achieve double robustness and fast parametric convergence rates. 


\begin{theorem}
\label{th:sto}
With the stochastic intervention of degree $\delta$ on observed data $z=(\boldsymbol{x},y,t)$, we have
\begin{equation}
    \varphi(z,\delta)=q_{t}(\boldsymbol{x},\delta)\cdot m_1(\boldsymbol{x},y)+(1-q_{t}(\boldsymbol{x},\delta))\cdot m_0(\boldsymbol{x},y)
    \label{eq:varphi}
\end{equation}
being the efficient influence function for the resulting counterfactual outcome $\hat{\psi}$, i.e., 
\begin{equation}
    \hat{\psi}
    =\mathbb{P}_{n}\left[\varphi(z,\delta)\right]
\label{eq:e_psi}
\end{equation}
where $m_1(\boldsymbol{x},y)$ or $m_1(\boldsymbol{x},y)$ is given by
\begin{equation}
    m_t(\boldsymbol{x},y)=\frac{\mathbb{I}_{t}\cdot({y-\hat{\mu}(\boldsymbol{x},t)})}{t\cdot\hat{p}_t(\boldsymbol{x})+(1-t)(1-\hat{p}_t(\boldsymbol{x}))}+\hat{\mu}(\boldsymbol{x},t)
    \label{eq:m}
\end{equation}
and $\mathbb{I}_{t}$ is an indicator function, $\hat{p}_t$ is the estimated propensity score in Eq.~\eqref{eq:ourps} and $\hat{\mu}$ is potential outcomes model that can be fitted by machine learning methods.
\end{theorem}



\begin{proof}


Throughout we assume the observed data quantity $\psi$ can be estimated under the positivity assumption from Section~\ref{sc:asp}.
For the unknown ground-truth $\psi(\delta)$, we will prove $\varphi$ is the influence function of $\psi(\delta)$ in Eq.~\eqref{eq:e_psi} by checking 
\begin{equation}
    \int \hat{\psi}(y,x,t,\mathbb{P}) d \mathbb{P}=\int\left(\varphi(y,x,t,\delta)-\psi \right) d \mathbb{P}=0
    \label{eq:property}
\end{equation} Eq.~\eqref{eq:property} indicates that the uncentered influence function $\varphi$ is unbiased for $\psi$. Given $q_t(\boldsymbol{x},\delta)$ as the stochastic propensity score in Eq.~\eqref{eq:incrps},
we check the property~\eqref{eq:property} by
\begin{equation*}
\small
\begin{split}
    &\int\left(\varphi(y,x,t,\delta)-\psi\right) d \mathbb{P}\\
    &=\int\left\{q_t\cdot m_1(\boldsymbol{x},y)+(1-q_t)m_0(\boldsymbol{x},y)-\psi(\delta)\right\}d\mathbb{P}(y,x,t,\delta)\\
    &=\int\{q_t\frac{\mathbb{I}_{t=1}\cdot(y-\hat{\mu}(x,1))}{\hat{p}_t}+(1-q_t)\frac{\mathbb{I}_{t=0}\cdot(y-\hat{\mu}(x,0))}{1-\hat{p}_t}\\ &+q_t\hat{\mu}(x,1)+(1-q_t)\hat{\mu}(x,0)-\psi(\delta)\}d\mathbb{P}(y,x,t,\delta)\\
    &=\int\{q_t\frac{\mathbb{I}_{t=1}\cdot(y-\hat{\mu}(x,1))}{\hat{p}_t}+(1-q_t)\frac{\mathbb{I}_{t=0}\cdot(y-\hat{\mu}(x,0))}{1-\hat{p}_t} \\ &+q_t\hat{\mu}(x,1)+(1-q_t)\hat{\mu}(x,0)-\mathbb{E}[q_t\hat{\mu}(x,1)\\
    &+(1-q_t)\hat{\mu}(x,0)]\}d\mathbb{P}(y,x,t,\delta)\\
    &\overset{(1)}{=}\int\left\{q_t\frac{\mathbb{I}_{t=1}\cdot(y-\hat{\mu}(x,1))}{\hat{p}_t}\right\}d\mathbb{P}(y,x,t,\delta)\\ &+\int\left\{(1-q_t)\frac{\mathbb{I}_{t=0}\cdot(y-\hat{\mu}(x,0))}{1-\hat{p}_t}\right\}d\mathbb{P}(y,x,t,\delta)\\
    &=\int\left\{q_t\frac{\mathbb{I}_{t=1}\cdot y}{\hat{p}_t}+(1-q_t)\frac{\mathbb{I}_{t=0}\cdot y}{1-\hat{p}_t}\right\}d\mathbb{P}(y,x,t,\delta)\\
    &-\int\left\{q_t\frac{\mathbb{I}_{t=1}\cdot\hat{\mu}(x,1)}{\hat{p}_t}-(1-q_t)\frac{\mathbb{I}_{t=0}\cdot\hat{\hat{\mu}}(x,0)}{1-\hat{p}_t}\right\}d\mathbb{P}(x,t,\delta)\\
    &\overset{(2)}{=}0
\end{split}
\end{equation*}
The second equation (1) follows from the iterated expectation, and the second equation (2) follows from the definition of $\hat{\mu}(\boldsymbol{x},t)$ and the usual properties of conditional distribution $d\mathbb{P}(x,y,\delta)=d\mathbb{P}(y|x,\delta)d\mathbb{P}(x,\delta)$.
So far we have proved that $\varphi$ is the influence function of average treatment effect $\psi(\delta)$. 
We have proved that the uncentered efficient influence function can be used to construct unbiased semiparametric estimator for $\psi(\delta)$, i.e., that 
$\int \varphi\mathbb{P}=\psi$.
\end{proof}
\begin{algorithm}[!htbp]
\small
\caption{SIE: Stochastic Intervention Effect}
\label{alg:avg}
\begin{algorithmic}[1]
\renewcommand{\algorithmicrequire}{\textbf{Input:}}
\renewcommand{\algorithmicensure}{\textbf{Output:}}
\REQUIRE Observed units $\{z_i:(\boldsymbol{x}_i,t_i,y_i)\}_{i=1}^{n}$
\STATE Initialize a stochastic degree $\delta$. 
\STATE Randomly split $Z$ into $k$ disjoint groups
\WHILE{each group}
\STATE Fit the propensity score $\hat{p}_t(\boldsymbol{x}_i)$ by Eq.~\eqref{eq:ourps} 

\STATE Fit the potential outcome model $\hat{\mu}(\boldsymbol{x}_i,t_i)$
\STATE Compute $\tau_{i}=\hat{p}_{t}(\boldsymbol{x}_i) \hat{\mu}(\boldsymbol{x}_i,1)+\left(1-\hat{p}_{t}(\boldsymbol{x}_i)\right) \hat{\mu}(\boldsymbol{x}_i,0)$
\STATE Calculate $q_{t}( \boldsymbol{x}_i;\delta)$ by Eq.~\eqref{eq:incrps}
\STATE Calculate $m_{1}(\boldsymbol{x}_i)$ and $m_{0}(\boldsymbol{x}_i)$ by Eq.~\eqref{eq:m}
\STATE Calculate the influence function $\varphi(z_i,\delta)$ by Eq.~\eqref{eq:e_psi}.
\ENDWHILE
\STATE Compute $\hat{\tau}_{\text{ATE}}=\frac{1}{n} \sum_{i=1}^{n}\tau_{i}$
\STATE Compute $\hat{\tau}_{\text{SIE}}=\frac{1}{n} \sum_{i=1}^{n}(\varphi(z_i,\delta)-y_i)$

\ENSURE stochastic intervention effect $\tau_{\text{SIE}}$.
\end{algorithmic}
\end{algorithm}
\subsection{Asymptotic Behavior Analysis}

Theorem~\ref{th:sto} ensures that the counterfactual outcome $\hat{\psi}$ can be estimated by its influence function $\varphi$ that depends on the nuisance function $(\hat{\mu}(\cdot),\hat{p}_t(\cdot))$.
We further analyze the asymptotic behavior of the influnence function-based estimator $\hat{\psi}$ to prove that $\hat{\psi}$ attains robustness even if $\hat{\mu}$ is mis-specified. With this theorem, we can claim that our semiparametric estimator is robust to the estimation of $(\hat{\mu}(\cdot),\hat{p}_t(\cdot))$. This is crucial for incorporating machine learning into our stochastic causal inference framework. 

\begin{theorem}
\label{lm:error}
The stochastic outcome estimator $\hat{\psi}$ in Eq.~\eqref{eq:e_psi} is asymptotically linear with influence function $\psi$, i.e., 
\begin{equation}
    \hat{\psi}-\psi=\mathbb{P}_{n}\{\varphi(z);\eta\}+o_{p}(1 / \sqrt{n})
\end{equation}
\end{theorem}

\begin{proof}
For notation simplicity, we use $z=(y,x,t,\delta)$ and $\eta=(\mu(\cdot),p_t(\cdot))$ for the true estimators in the proof.  
Suppose the estimator $\hat{\eta}=(\hat{\mu},\hat{p}_t)$ converges to some $\overline{\eta}=(\overline{\mu},\overline{p}_t)$ in the sense that $\|\hat{\eta}-\overline{\eta}\|=o_p(1)$, where either $\overline{p}_t=p_t$ or $\overline{\mu}=\mu$ correspond to the true nuisance function.
Therefor, we conclude that at least one nuisance estimator needs to converge to the correct function, but the other one can be misspecified. 
We denote the mispecified functions $\tilde{\mu}$ and $\tilde{p}_t$ in the neighborhood of $\mu$ and $p_t$, respectively.
From the fact that $\mathbb{P}\left\{\varphi\left(z ; p_t, \tilde{\mu}\right)\right\}=\mathbb{P}\left\{\varphi\left(z ; \tilde{p}_t, \mu\right)\right\}$, we have $\mathbb{P}\{\varphi)(z;\overline{\eta}\}=\mathbb{P}\{\varphi)(z;\eta\}=\psi$ for any $\overline{p}_t$ and $\overline{\mu}$. We can write
\begin{equation}
    \hat{\psi}-\psi=\left(\mathbb{P}_{n}-\mathbb{P}\right) \varphi(z ; \hat{\eta})+\mathbb{P}\{\varphi(z ; \hat{\eta})-\varphi(z ; \bar{\eta})\}
    \label{eq:4.17}
\end{equation}
If $\hat{\mu}$ and $\hat{p}_t$ are usual parametric functions in Donsker classes~\cite{dudley2010universal}, then $\varphi(z ; \hat{\eta})$ is enabled with Donsker property, i.e.,
\begin{equation}
    \left(\mathbb{P}_{n}-\mathbb{P}\right) \varphi(z ; \hat{\eta})=\left(\mathbb{P}_{n}-\mathbb{P}\right) \varphi\left(z ; \eta\right)+o_{p}(1 / \sqrt{n})
    \label{eq:donsker_m}
\end{equation}
Substitute Eq.~\eqref{eq:donsker_m} to Eq.~\eqref{eq:4.17}, we have
\begin{equation}
    \hat{\psi}-\psi=\left(\mathbb{P}_{n}-\mathbb{P}\right) \varphi(z ; \eta)+\mathbb{P}\{\varphi(z ; \hat{\eta})-\varphi(z ; \bar{\eta})\}+o_{p}(1 / \sqrt{n})
    \label{eq:4.19}
\end{equation}
The iterated expectation of term $\mathbb{P}\{\varphi(z ; \hat{\eta})-\varphi(z ; \bar{\eta})\}$ in Eq.~\eqref{eq:4.19} equals
\begin{equation}
\small
    \sum_{t \in\{0,1\}} \mathbb{P}\left[\frac{p_t(\boldsymbol{x})-\hat{p_t}(\boldsymbol{x})}{t\cdot \hat{p_t}(\boldsymbol{x})+(1-t)\{1-\hat{p_t}(\boldsymbol{x})\}}\left\{\mu(\boldsymbol{x}, t)-\hat{\mu}(\boldsymbol{x}, t)\right\}\right]
\end{equation}
According to the fact that $0<\hat{p}_t<1$ and the Cauchy-Schwarz inequality $\mathbb{P}(f\cdot g)\leq \|f\|\|g\|$, then $\mathbb{P}\{\varphi(z ; \hat{\eta})-\varphi(z ; \bar{\eta})\}\leq$
\begin{equation}
     \sum_{t \in\{0,1\}}\left\|p_t(\boldsymbol{x})-\hat{p}_t(\boldsymbol{x})\right\|\left\|\mu(\boldsymbol{x}, t)-\hat{\mu}(\boldsymbol{x}, t)\right\|
\end{equation}
Therefore, if $\hat{p}_t$ a correct parametric model for propensity score, so that $\|\hat{p}_t-p_t\|=o_p(\frac{1}{\sqrt{n}})$, then we only need $\hat{\mu}$ to be consistent, $\|\hat{\mu}-\mu\|=o_p(1)$ to allow $\mathbb{P}\{\varphi(z ; \hat{\eta})-\varphi(z ; \bar{\eta})\}=o_{p}(\frac{1}{\sqrt{n}})$ asymptotically negligible. Then our influence-based estimator satisfied $\hat{\psi}-\psi=\left(\mathbb{P}_{n}-\mathbb{P}\right) \varphi\left(z ; \eta\right)+o_{p}(\frac{1}{\sqrt{n}})$. 

\end{proof}
According to Theorem~\eqref{lm:error}, if the propensity score model in Eq.~\eqref{eq:ourps} is unbiased, the potential outcome model can be estimated by $\hat{\psi}$ in a flexible manner. Because the influence function we defined contains all information about an estimator’s asymptotic behavior (up to $o_{p}(1 / \sqrt{n})$ error).


\section{Stochastic Intervention Optimization}
Estimating the stochastic intervention effect is not enough; we are more interested in ``what is the optimal level/degree of treatment for a patient to achieve the most expected outcome?''. A direct way to find the optimal treatment is through reinforcement learning, which focuses on finding policy/intervention for controlling dynamical systems with the goal of maximizing the desired outcome on downstream decision-making tasks. This is done by the agent repeatedly observing its state, taking action (according to a policy/intervention), and receiving a reward. Over time the agent modifies its policy to maximize its long-term desired outcome.
In this paper, we focus particularly on model-free reinforcement learning algorithms, which have become popular in offering off-the-shelf solutions without requiring models of the system dynamics~\cite{feinberg2018model,feinberg2018model}. 
However, the intervention is stochastic rather than deterministic, which tends to result in large training variances in action space. Handling large variance is a significant challenge in model-free reinforcement learning (RL)~\cite{cheng2019control}, which would result in the degenerate performance in the intervention optimization.

To alleviate the aforementioned issue, we consider the basic random search method, which explores in the parameter space rather than the action space and thus achieves the optimal expected outcome in a more efficient manner.
We model the stochastic intervention using the stochastic propensity score $\hat{q}_t(\boldsymbol{x},\delta)$, and 
look for the optimal stochastic interventions parameter $\Delta^* \in \mathbb{R}^{n\text{x}1}$ such that:

\begin{equation}
    \Delta^* =\arg\max_{\Delta}\sum_{i=1}^{n}\varphi(z_i,\Delta
    )
    \label{eq:delta}
 \end{equation}
\label{eq:strategy}

Note that the optimization problem in Eq.~\eqref{eq:delta} is non-differentiable.
To avoid using further assumptions for solving it, we formulate a customised reinforcement learning algorithm~\cite{whitley1994genetic} (RS-SIO) to exploit the search space. 
The main advantage of RS-SIO is model-agnostic which can handle with any black-box functions and flexibly deal with any data type including continuous and categorical features. Therefore, with modifications specific to the intervention effect estimation, RS-SIO solves Eq.~\eqref{eq:strategy} through the discovery process of trial-and-error search \cite{qiang2011reinforcement, whitehead1991learning, barto1995reinforcement} which gradually updates the stochastic parameters in every step based on the rewards. Particularly, the algorithm firstly initializes the stochastic intervention parameter $\Delta_0 = \boldsymbol{0} \in \mathbb{R}^{n\text{x}1}$ and samples a set of $\boldsymbol{\delta}$ having the same size as $\Delta_0$. Thereafter, for each $\boldsymbol{\delta}$, we compute the rewards when the search process moves toward to the positive ($\Phi(\boldsymbol{\delta_k}, +)$) and negative direction ($\Phi(\boldsymbol{\delta_k}, -)$), and then select the number of $b$ largest awards for these directions as $\text{max}\{\Phi(\boldsymbol{\delta_k}, +), \Phi(\boldsymbol{\delta_k}, -)\}$. In order to update the stochastic parameters $\Delta$, we exploit the update directions $\frac{1}{b}\sum_{k=1}^b[\Phi(\boldsymbol{\delta_k}, +) - \Phi(\boldsymbol{\delta_k}, -)]\boldsymbol{\delta_k}$. The full stochastic intervention optimization algorithm is shown in Algorithm~\ref{alg:IEO}. 

\begin{algorithm}[H]
\small
\caption{Random Search based Reinforcement Learning for SIO (RS-SIO)}
\label{alg:IEO}
\begin{algorithmic}[1]
\renewcommand{\algorithmicrequire}{\textbf{Input:}}
\renewcommand{\algorithmicensure}{\textbf{Output:}}
\REQUIRE Observed units $\{z_i:(\boldsymbol{x}_i,t_i,y_i)\}_{i=1}^{n}$, step-size $\alpha$, standard deviation of the exploration noise $\varv$, number of steps $l$, number of top-performing directions $b$. 

\STATE Initialize the stochastic intervention parameter ${\Delta}_0=\boldsymbol{0} \in \mathbb{R}^{n\text{x}1}$
\STATE Sample $\boldsymbol{\delta}_1$, $\boldsymbol{\delta}_2$,.., $\boldsymbol{\delta}_m$ of the same size as  ${\Delta}_0$.

\FOR{$j=1,\cdots,l$}
\FOR {$k=1,\cdots,m$}
    \FOR{$i=1,\cdots,n$}
    \STATE Compute $q_t(\boldsymbol{x_i},\boldsymbol{\delta_k})$ by Eq.~\eqref{eq:incrps}
    \STATE Compute $m_{1}(\boldsymbol{x}_i)$ and $m_{0}(\boldsymbol{x}_i)$ by Eq.~\eqref{eq:m}
    \STATE Compute $\varphi(z_i,\Delta_j + \varv \boldsymbol{\delta_k})$ and $\varphi(z_i,\Delta_j - \varv \boldsymbol{\delta_k})$ by Eq.~\eqref{eq:varphi}. 
    \ENDFOR
\ENDFOR
\FOR {$k=1,\cdots,m$}
\STATE Compute the reward 
\begin{equation}
\begin{split}
  \Phi(\boldsymbol{\delta_k}, +)=\sum_{i=1}^{n}\varphi(z_i,\Delta_j + \varv \boldsymbol{\delta_k}),\quad
\Phi(\boldsymbol{\delta_k}, -)=\sum_{i=1}^{n}\varphi(z_i,\Delta_j - \varv \boldsymbol{\delta_k})  
\end{split}
\end{equation}
\ENDFOR

\STATE Sort $\boldsymbol{\delta_k}$ by  $\text{max}\{\Phi(\boldsymbol{\delta_k}, +), \Phi(\boldsymbol{\delta_k}, -)\}$ and select $b$ top-performing directions.


\STATE Update 
\begin{equation}
    \Delta_{j+1} = \Delta_j + \frac{\alpha}{b}\sum_{k=1}^b[\Phi(\boldsymbol{\delta_k}, +) - \Phi(\boldsymbol{\delta_k}, -)]\boldsymbol{\delta_k} 
\end{equation}

\ENDFOR


\ENSURE $\Delta_j$
\end{algorithmic}
\end{algorithm}

\section{Experiments and Results}
\label{section:experiment}
In this section, we conduct intensive experiments and compare our framework with state-of-the-art methods on two tasks: treatment effect estimation and stochastic intervention effect optimization. Recall that the influence-based estimator $\varphi$ depends on the nuisance function of propensity score $p_t$ and outcome $\mu$. We first perform average treatment effect estimation to confirm that $\hat{p}_t$ and $\hat{\mu}$ are unbiased and robust estimators. 
Moreover, the stochastic intervention optimization task is carried out to demonstrate the effectiveness of our RS-SIO.


\subsection{Baselines}
\label{sec:base}

We briefly describe the comparison methods which are used in two tasks of treatment effect estimation and stochastic intervention optimization.

Evaluating the performance of SIE is not an easy task, because the ground-truth counterfactual outcome is unobserved in practice. On the contrary, the benchmark datasets having two potential outcomes are available for ATE estimation. Therefore, we perform ATE estimation to evaluate the robustness of $\hat{p}_t$ and $\hat{\mu}$ thus to indirectly evaluate the performance of SIE.
We use Gradient Boosting Regression with 100 regressors for the potential outcome models $\hat{\mu}$.
We compare our proposed estimator (SIE) with the following baselines including Doubly Robust Leaner~\cite{10.5555/3104482.3104620} (LinearDRLearner and ForestDRLearner), IPWE~\cite{austin2015moving}, BART~\cite{hill2011bayesian}, Causal Forest~\cite{wager2018estimation, athey2019generalized}, TMLE~\cite{gruber2011tmle} and OLS~\cite{goldberger1964econometric}. 
Regarding implementation and parameters setup, we adopt Causal Forest~\cite{wager2018estimation, athey2019generalized} with 100 trees, BART \cite{hill2011bayesian} with 50 trees and TMLE~\cite{gruber2011tmle} from the libraries of cforest, pybart and zepid in Python. For Doubly Robust Learner (DR) \cite{10.5555/3104482.3104620}, we use the two implementations, i.e. LinearDRL and ForestDRL from the package EconML~\cite{econml} with Gradient Boosting Regressor with 100 regressors as the regression model, and Gradient Boosting Classifier with 200 regressors as the propensity score model. Ultimately, we use package DoWhy \cite{dowhy} for IPWE \cite{austin2015moving} and OLS.

For stochastic intervention optimization, we compare our proposed method (RS-SIO) with Separate Model Approach (SMA) with different settings. SMA is a uplift  modeling  method that estimates  the  user-level  incremental  effect  of  a  treatment using  machine  learning  models. SMA \cite{zaniewicz2013support,Personalized_Medicine} aims to build two separate regression models for the outcome prediction in the treated and controlled group, respectively. Under the setting of SMA, we apply four well-known models for predicting outcome including Random Forest (SMA-RF) \cite{soltys2015ensemble, grimmer2017estimating}, Gradient Boosting Regressor (SMA-GBR) \cite{friedman2001greedy}, Support Vector Regressor (SMA-SVR) \cite{zaniewicz2013support}, and AdaBoost (SMA-AB) \cite{solomatine2004adaboost}. We also compare the performance of these models with the random policy to justify that optimization algorithms can help to target the potential customers to generate greater revenue. 
For the settings of SMA, we use Gradient Boosting Regressor with 1000 regressors, AdaBoost Regression with 50 regressors, and Random Forest Tree Regressor with 100 trees.

\subsection{Datasets} 
\texttt{IHDP} \cite{hill2011bayesian} is a standard semi-synthetic dataset used in the \emph{Infant Health and Development Program}, which is a popularly used semi-synthetic benchmark containing both the factual and counterfactual outcomes. We conduct the experiment on 100 simulations of \texttt{IHDP} dataset, in which each dataset is divided into training and testing set\footnote{\url{http://www.fredjo.com/files/ihdp_npci_1-100.train.npz} and \url{http://www.fredjo.com/files/ihdp_npci_1-100.test.npz}}. The training dataset is highly imbalanced with 139 treated and 608 controlled units out of 747 units, respectively, whilst the testing dataset has 75 units. Each unit has 25 covariates representing the individuals' characteristics. The outcomes are their IQ scores at age three~\cite{dorie2016npci}. 


Online promotion dataset (\texttt{OP} Dataset) provided by EconML project \cite{econml} is chosen to evaluate stochastic intervention optimization\footnote{\url{https://msalicedatapublic.blob.core.windows.net/datasets/Pricing/pricing_sample.csv}}. This dataset consists of 10k records in online marketing scenario with the treatment of discount price and the outcome of revenue, each represents a customer with 11 covariates. We split the data into two parts: 80\% for training and 20\% for testing set. We run 100 repeated experiments with different random states to ensure the model outcome reliability. We aim to investigate how to maximize the revenue by applying different price policies to different customers. 

\texttt{Lalonde} \footnote{\url{https://users.nber.org/~rdehejia/data/.nswdata2.html}} \cite{dehejia1999causal, dehejia2002propensity} is the real-world dataset about the men in the \emph{National Supported Work Demonstration} who were or were not provided the on-job training for more than nine months. Each unit has six features, including age, education, black (1 if black, 0 otherwise), Hispanic (1 if Hispanic, 0 otherwise), married and degree. The outcomes are their earnings in 1975 and 1978 with 297 treated and 425 control observations. The main goal of this dataset is to determine the monetary benefits of the job training on the people. For this dataset, we conduct experiments to find the optimal policy such that their earnings in 1975 and 1978 are maximized. We also repeat experiments 100 times with different random states to ensure model stability. 

\subsection{Evaluation Metrics}
In this section, we briefly describe the two evaluation metrics used for stochastic intervention effect estimation and stochastic intervention optimization, respectively.
\begin{itemize}
    \item \textbf{Stochastic Intervention Effect Estimation}. Based on average treatment effect (ATE) in Eq.~\eqref{eq:tau}, we evaluate the performance of treatment effect estimation by the mean absolute error between the estimated and true ATE: 
\begin{equation}
    \epsilon_{ATE} = |\hat{\tau}_{\text{ATE}} - \tau_{\text{ATE}}|
\end{equation}
\item \textbf{Stochastic Intervention Optimization}. 
Followed by the uplifting model studies~\cite{zhao2017uplift, hitsch2018heterogeneous}, we use the expected value of the outcome under the policy proposed by the models as the main metric, which can be measured as:
\begin{equation}
\label{eq:erupt}
\begin{split}
  \hat{y} &= \mathbb{E}[\boldsymbol{y} | \boldsymbol{t} = \pi(\boldsymbol{x})] \\ &= \frac{1}{n}\sum_{i=0}^n\sum_{a=0}^1 \frac{1}{p_i}y_{i}\mathbb{I}\{t_{i} = \pi(x_{i})\}\mathbb{I}\{t_{i}=a\}
\end{split}
\end{equation}

where $p_i$ is the propensity score of individual $i$, $\mathbb{I}\{.\}$ is the indicator function with 1 for true condition and 0 otherwise and $\pi(\boldsymbol{x})$ is the proposed policy. Our method uses $\pi(\boldsymbol{x})=\Delta(\boldsymbol{x})$ with $\Delta$ is the stochastic intervention parameter. Particularly, the expected outcome is computed as if the predicted treatment matches the current treatment, the expected outcome is scaled by the inverse of propensity score $y_{i} / p_{i}$, otherwise, equals zero.

\end{itemize}




\subsection{Results and Discussions}

In this section, we aim to report the experimental results of 1) how our proposed estimator (SIE) can accurately estimate the average treatment effect; 2) how our optimization algorithm (RS-SIO) can be used for finding optimal stochastic intervention in real-world datasets.

\subsubsection{Treatment Effect Estimation}
The results of $\epsilon_{ATE}$ derived from \texttt{IHDP} dataset with 100 simulations and \texttt{OP} dataset with 100 repeated experiments are presented in the Table~\ref{table:ihdp}. As seen clearly, amongst all approaches, our proposed estimator SIE achieves the best performance under $\epsilon_{ATE}$ in both two datasets, followed by TMLE for \texttt{IHDP} dataset, and LinearDRL and ForestDRL for \texttt{OP} dataset. Particularly, on \texttt{IHDP}, SIE outperforms all other methods in both training and testing sets. In order to investigate the impact of data size chosen on estimation, we also run experiments and plot the performance of models in different data sizes in Figure~\ref{fig:ihdp}. Notably, SIE consistently produces the more accurate average treatment effect than others as the data size increases, while TMLE is ranked second in this dataset. LinearDRL and Causal Forest also produce very competitive results, whereas IPWE performs the worst. For the experimental results on the online promotion dataset, SIE consistently achieves the outstanding performance under $\epsilon_{ATE}$, followed by the performance of LinearDRL and ForestDRL, while the competitive results are also recorded with BART. It is also worthy to note that although TMLE performs well in the training set, its performance likely degrades when dealing with out-of-sample data in the testing set. Regarding the computational time, IPWE and BART are the best-performing and worst-performing, respectively. Our proposed method is ranked third among the baselines for both two datasets, which is acceptable 
in consideration of our superior performances on treatment effect estimation compared to IPWE and BART.

\begin{table}[!htb]	
	\centering
	\small
	\caption{$\epsilon_{ATE}$ and running time of baselines on \texttt{IHDP}(lower is better).}
	
	\resizebox{\columnwidth}{!}{
	\begin{tabular}{ccccccc}
		\multirow{2}[3]{*}{Method} & \multicolumn{3}{c}{\texttt{IHDP} Dataset} & \multicolumn{3}{c}{\texttt{OP} Dataset} \\
\cmidrule(lr){2-4}  \cmidrule(lr){5-7} 

& Train & Test & Time (ms) & Train & Test & Time (s)\\

\hline
		OLS & 0.746 $\pm$ 0.140 & 1.264 $\pm$ 0.250  & 242.498 $\pm$ 0.000 & 5.906 $\pm$ 0.004 & 5.906 $\pm$ 0.004 &   8.891 $\pm$ 0.000  \\
BART & 1.087 $\pm$ 0.120 & 2.808 $\pm$ 0.100 & 2353.843 $\pm$ 0.000   & 0.504 $\pm$ 0.042 &   0.505 $\pm$ 0.043 &   14.180 $\pm$ 0.000\\
Causal Forest & 0.360 $\pm$ 0.050 & 0.883  $\pm$ 0.614 & 180.100 $\pm$ 0.000  & 3.520 $\pm$ 0.034 &    3.520 $\pm$ 0.034 &   5.907 $\pm$ 0.000\\
TMLE & 0.326 $\pm$ 0.060 & 0.831 $\pm$ 1.750 & 584.659 $\pm$ 0.000 & 0.660 $\pm$ 0.000 &   3.273 $\pm$ 0.000  &   9.723 $\pm$ 0.000\\
ForestDRLearner & 1.044 $\pm$ 0.040 & 1.224 $\pm$ 0.080 & 241.148 $\pm$ 0.000 & 0.240 $\pm$ 0.014 &   0.241 $\pm$ 0.013 &   7.807 $\pm$ 0.000\\
LinearDRLearner & 0.691 $\pm$ 0.080 & 0.797 $\pm$ 0.170 & 269.193 $\pm$ 0.000 & 0.139 $\pm$ 0.009 &    0.139 $\pm$ 0.008  &   7.107 $\pm$ 0.000\\
IPWE & 1.701 $\pm$ 0.140 &  5.897  $\pm$ 0.300 & 84.531 $\pm$ 0.000 & 5.908 $\pm$ 0.000 &    5.908 $\pm$ 0.015 &   2.1725 $\pm$ 0.000\\
\hline
SIE & \textbf{0.284 $\pm$ 0.050} & \textbf{0.424 $\pm$ 0.090 }  & 200.135 $\pm$ 0.000 &   \textbf{0.137 $\pm$ 0.000} &   \textbf{0.119 $\pm$ 0.000} &   7.002 $\pm$ 0.000\\
\hline

	\end{tabular}}
	\label{table:ihdp}
\end{table}

\begin{figure}[!htb]
\label{fig:ihdp}
\centerline{\includegraphics[width=\textwidth]{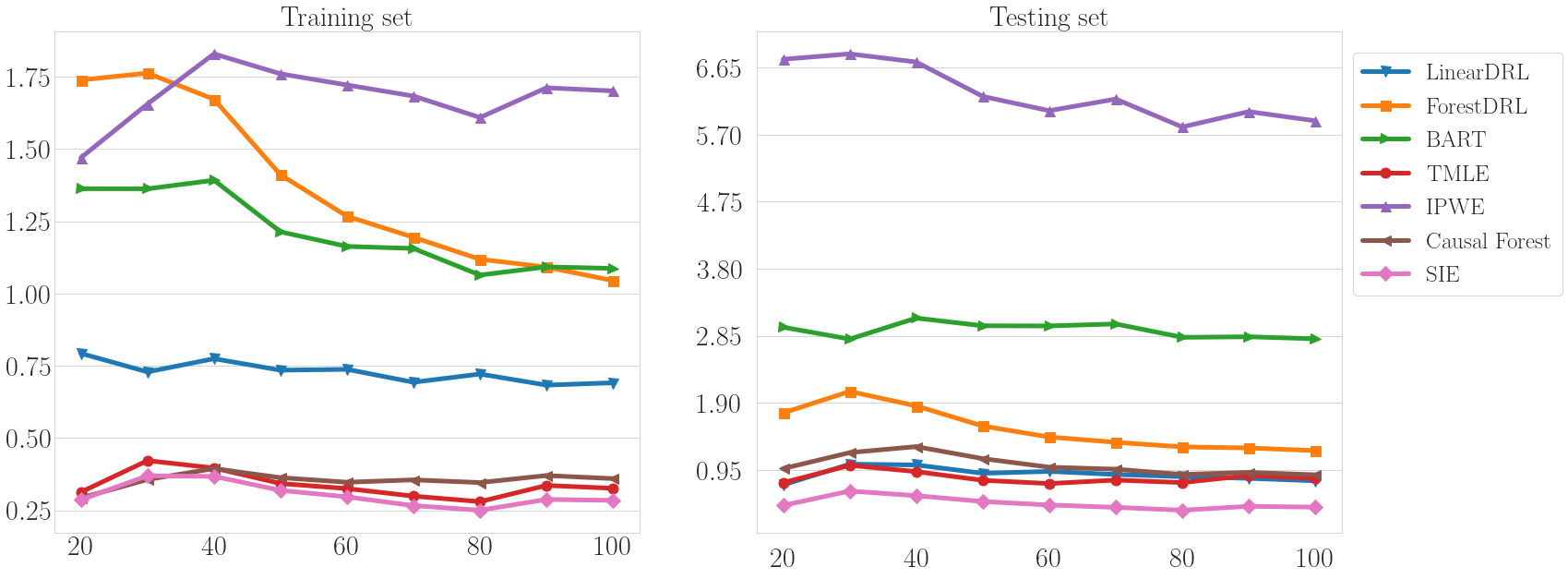}}
\caption{$\epsilon_{ATE}$ of baselines on \texttt{IHDP} dataset with different samples.}
\label{fig:ihdp}
\end{figure}

\subsubsection{Stochastic Intervention Optimization}
For the online promotion dataset, we model the revenue as the expected outcome of each customer under the policy/intervention.
Figure~\ref{fig:revenue} presents the revenue of uplifting modeling methods with different data sizes including 1000, 5000 and 10000 records. 
We set step $l=100$ for our proposed method (RS-SIO). RS-SIO compares favorably to recent uplift modeling techniques that
optimize the policy (or intervention) on treatment to maximize the expected outcome.
Apparently, RS-SIO generally produces the highest revenue in datasets with different samples, while SMA-ABR achieves the second-best performance with a very competitive result. Moreover, we find that no significant difference in the performance of SMA with different settings. In contrast, random stochastic intervention produces the lowest revenue, which fails to target the customers for the promotion. On the other hand, Figure~\ref{fig:earning} illustrates the predicted earnings in 1975 and 1978 by different methods. As can be seen, the maximum earning is produced by our proposed method, while random policy/intervention produces the lowest earnings in 1975 and 1978. SMA-Ridge and SMA-GBR achieve competitive performance in this dataset. The possible reason behind our outstanding performance is that instead of focusing on predicting the outcomes like SMA, we directly intervene into the propensity score to produce the best stochastic intervention. 

\begin{figure}[!htb]
\centerline{\includegraphics[width=0.65\textwidth]{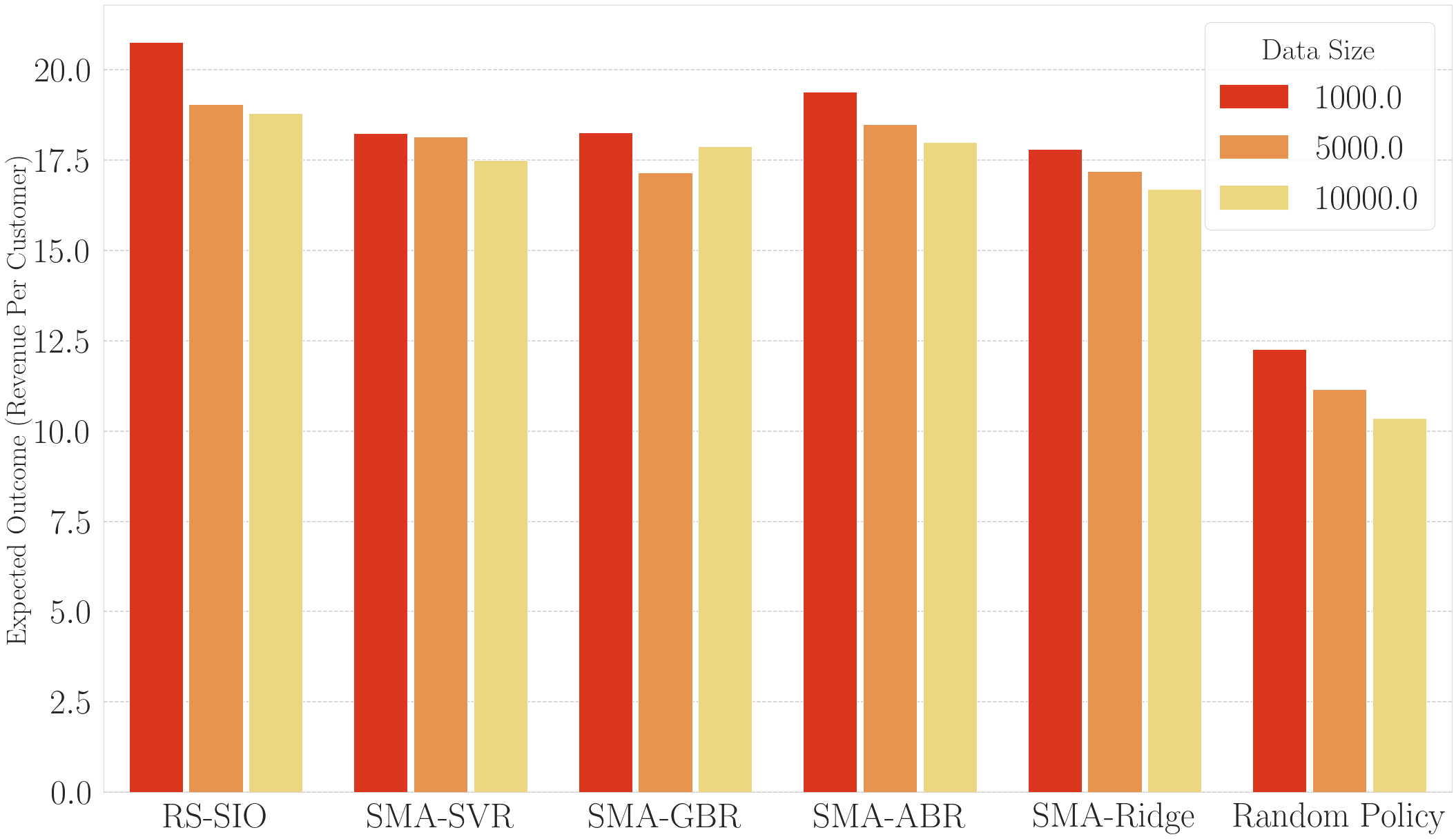}}
\caption{Intervention optimization on \texttt{OP} dataset by different baselines.}
\label{fig:revenue}
\end{figure}

\begin{figure}[!htb]
\includegraphics[width=\textwidth]{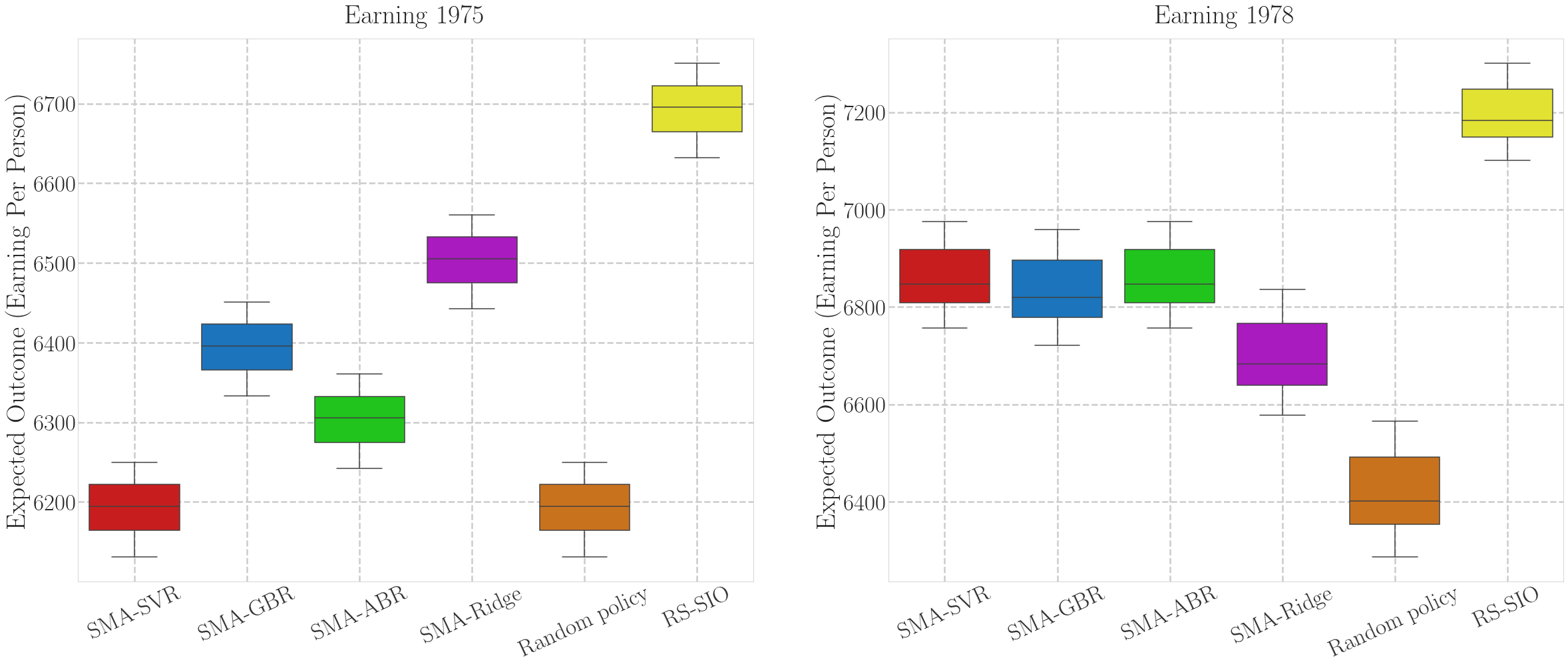}
\caption{Intervention optimization on {\texttt{Lalonde}} dataset by different baselines.}
\label{fig:earning}
\end{figure}

\section{Conclusion}
\label{section:conclusion}
We have developed a causal inference framework that admits the stochastic intervention in treatment effect estimation and designs an effective causal solution for the intervention effect optimization. 
In general, the contribution of this study is twofold. Firstly, we propose a novel treatment effect estimator based on stochastic propensity score, which can effectively learn the trajectory of the stochastic intervention effect. Secondly, we design a reinforcement learning algorithm to find the optimal intervention for maximizing the expected outcome, thus providing causal insights for an effective decision-making process. We provide theoretical guarantees for the stochastic intervention effect estimator to achieve double robustness and fast parametric convergence rates.      
Extensive numerical results justify that our framework outperforms state-of-the-art baselines in both treatment effect estimation and stochastic intervention optimization.

One limitation of our causal framework is that the stochastic intervention is set to static data,  i.e., the observational data are time-independent. In many real-world applications, however, events change over time, e.g., each unit may receive a stochastic intervention multiple times, and the timing of these interventions may differ across units\cite{kennedy2017nonparametric, hill2011bayesian, galagate2016causal}.
Of practical interest is to perform a more detailed empirical study on the time-dependent stochastic intervention. 

\bibliography{mybibfile}

\end{document}